\def\year{2020}\relax
\documentclass[letterpaper]{article} 
\usepackage{aaai20}  
\usepackage{times}  
\usepackage{helvet} 
\usepackage{courier}  
\usepackage[hyphens]{url}  
\usepackage{graphicx} 
\usepackage{graphicx}  
\usepackage{amssymb}
\usepackage{tikz}
\usetikzlibrary{automata,positioning}
\usetikzlibrary{matrix}
\usepackage{marvosym}
\usepackage[marvosym]{tikzsymbols}
\usepackage[caption=false]{subfig}

\usepackage{amsmath}
\usepackage{mathtools}
\usepackage{booktabs}
\usepackage{algorithm}
\usepackage[]{algpseudocode}
\usepackage{siunitx}
\usepackage{verbatim}
\usepackage{amsthm}
\usepackage{comment}

\algtext*{EndWhile}
\algtext*{EndIf}
\algtext*{EndFor}
\algtext*{EndFunction}

\algnewcommand\algorithmicinput{\textbf{Input:}}
\algnewcommand\INPUT{\item[\algorithmicinput]}
\algnewcommand{\algorithmicbreak}{\textbf{break}}

\DeclareMathOperator{\codeif}{\mathtt{:-} }

\tikzset{%
	in place/.style={
		auto=false,
		fill=white,
		inner sep=2pt,
	},
}

\newtheorem{theorem}{Theorem}
\newtheorem{lemma}{Lemma}

\newcommand{\methodname}{ISA}

\newcommand{\methodnamefullhighlight}{\textbf{I}nduction of \textbf{S}ubgoal \textbf{A}utomata for Reinforcement Learning}

\title{Induction of Subgoal Automata for Reinforcement Learning}
\author{Daniel Furelos-Blanco,\textsuperscript{\rm 1} Mark Law,\textsuperscript{\rm 1} Alessandra Russo,\textsuperscript{\rm 1} Krysia Broda,\textsuperscript{\rm 1} Anders Jonsson\textsuperscript{\rm 2}\\ 
\textsuperscript{\rm 1}Imperial College London, United Kingdom,
\textsuperscript{\rm 2}Universitat Pompeu Fabra, Barcelona, Spain\\
\{d.furelos-blanco18, mark.law09, a.russo, k.broda\}@imperial.ac.uk,
anders.jonsson@upf.edu
}
 
\begin{document}

\maketitle

\begin{abstract}
In this work we present {\methodname}, a novel approach for learning and exploiting subgoals in reinforcement learning (RL). Our method relies on inducing an automaton whose transitions are subgoals expressed as propositional formulas over a set of observable events. A state-of-the-art inductive logic programming system is used to learn the automaton from observation traces perceived by the RL agent. The reinforcement learning and automaton learning processes are interleaved: a new refined automaton is learned whenever the RL agent generates a trace not recognized by the current automaton. We evaluate {\methodname} in several gridworld problems and show that it performs similarly to a method for which automata are given in advance. We also show that the learned automata can be exploited to speed up convergence through reward shaping and transfer learning across multiple tasks. Finally, we analyze the running time and the number of traces that {\methodname} needs to learn an automata, and the impact that the number of observable events has on the learner's performance.
\end{abstract}

\section{Introduction}
Reinforcement learning (RL) \cite{SuttonB98} is a family of algorithms where an agent acts in an environment with the purpose of maximizing the total amount of reward it receives. These algorithms have played a key role in major breakthroughs like human-level video game playing \cite{MnihKSRVBGRFOPB15} or mastering complex board games \cite{Silver18}. However, current methods lack the ability to generalize and transfer between tasks. For example, they cannot learn to play chess using the knowledge of sh$\bar{\text{o}}$gi.

Advances in achieving generalization and transfer in RL are mainly due to abstractions \cite{Ho2019,Konidaris2019}. Hierarchical reinforcement learning methods, which divide a single task into several subtasks that can be solved separately, are among the most promising approaches to abstracting RL tasks. Common frameworks for hierarchical RL are HAMs \cite{ParrR97}, options \cite{SuttonPS99} and MAXQ \cite{Dietterich00}.

Abstract hierarchies can be naturally represented using automata, which have been used effectively in task abstraction, not only in RL \cite{ParrR97,IcarteKVM18}, but also in related areas like automated planning \cite{BonetPG09,HuD11,SegoviaJJ18}. However, these RL approaches use handcrafted automata instead of learning them from data. It is largely an open problem in RL how to autonomously decompose a task into an automaton that can be exploited during the RL process.

In this paper, we address this problem and propose {\methodname} (\methodnamefullhighlight), a method for learning and exploiting a minimal automaton from observation traces perceived by an RL agent. These automata are called \emph{subgoal automata} since each transition is labeled with a subgoal, which is a boolean formula over a set of observables. Subgoal automata can be expressed in a logic programming format  and, thus, be learned using state-of-the-art inductive logic programming (ILP) systems \cite{ILASP_system}. The resulting subgoal automaton can be exploited by an RL algorithm that learns a different policy for each automaton state, each aiming to achieve a subgoal. This decomposition allows learning multiple subpolicies simultaneously and transferring knowledge across multiple tasks. {\methodname} interleaves the RL and automaton learning processes such that the agent can immediately leverage the new (partially correct) learned subgoal automaton: when a trace is not correctly recognized by the automaton, a new one is learned.

We evaluate {\methodname} in several gridworld problems and show that it learns subgoal automata and policies for each of these. Specifically, it performs similarly to a method for which automata are given in advance. Furthermore, the learned automata can be exploited to speed up convergence through reward shaping and transfer learning.

The paper is organized as follows. Section \ref{sec:background} introduces the background of our work. Section \ref{sec:methodology} formally presents the problem we address and our method for solving it, while Section \ref{sec:results} describes the results. We discuss related work in Section \ref{sec:related_work} and conclude in Section \ref{sec:conclusions}.

\section{Background}
\label{sec:background}
In this section we briefly summarize the key background concepts used throughout the paper: reinforcement learning and inductive learning of answer set programs.

\subsection{Reinforcement Learning}
Reinforcement learning (RL) \cite{SuttonB98} is a family of algorithms for learning to act in an unknown environment. Typically, this learning process is formulated as a \emph{Markov Decision Process (MDP)}, i.e., a tuple $\mathcal{M}=\langle S,A,p,r,\gamma\rangle$, where $S$ is a finite set of states, $A$ is a finite set of actions, $p:S\times A \to \Delta(S)$ is a transition probability function\footnote{For any finite set $X$, $\Delta(X) = \{\mu \in \mathbb{R}^X : \sum_x\mu(x) = 1, \mu(x) \geq 0~(\forall x) \}$ is the probability simplex over $X$. }, $r:S \times A \times S \to \mathbb{R}$ is a reward function, and $\gamma \in [0,1)$ is a discount factor. At time $t$, the agent observes state $s_t \in S$, executes action $a_t \in A$, transitions to the next state $s_{t+1} \sim p(\cdot|s_t, a_t)$ and receives reward $r(s_t, a_t, s_{t+1})$.

We consider {\em episodic} MDPs that \emph{terminate} in a given set of terminal states. We distinguish between goal states and undesirable terminal states (i.e.,~dead-ends). Let $S_T\subseteq S$ be the set of terminal states and $S_G\subseteq S_T$ the set of goal states. The aim is to find a \emph{policy} $\pi:S \to \Delta(A)$, a mapping from states to probability distributions over actions, that maximizes the expected sum of discounted reward (or \emph{return}), $R_t=\mathbb{E}[\sum_{k=t}^n\gamma^{k-t} r_k]$, where $n$ is the last episode step.

In model-free RL the transition probability function $p$ and reward function $r$ are unknown to the agent, and a policy is learned via interaction with the environment. Q-learning \cite{Watkins:1989} computes an \emph{action-value function} $Q^\pi(s,a) = \mathbb{E}[R_t|s_t=s, a_t=a]$ that estimates the return from each state-action pair when following an approximately optimal policy. In each iteration the estimates are updated as $$
Q\left(s,a\right) \xleftarrow{\alpha} r\left(s,a,s'\right) + \gamma\max_{a'}Q\left(s',a'\right),
$$ where $x\xleftarrow{\alpha}y$ is shorthand for $x \leftarrow x+ \alpha \left(y-x\right)$, $\alpha$ is a learning rate and $s'$ is the state after applying $a$ in $s$. Usually, an $\epsilon$-greedy policy selects a random action with probability $\epsilon$ and the action maximizing $Q(s,a)$ otherwise. The policy is induced by the action that maximizes $Q(s,a)$ in each $s$.

\subsubsection{Options}\cite{SuttonPS99} address temporal abstraction in RL. Given an MDP $\mathcal{M}=\langle S,A,p,r,\gamma\rangle$, an \emph{option}  is a tuple $\omega=\langle I_\omega, \pi_\omega, \beta_\omega\rangle$ where $I_\omega \subseteq S$ is the option's initiation set, $\pi_\omega:S \to \Delta(A)$ is the option's policy, and $\beta_\omega:S\to [0,1]$ is the option's termination condition. An option is available in state $s \in S$ if $s\in I_\omega$. If the option is started, the actions are chosen according to $\pi_\omega$. The option terminates at a given state $s\in S$ with probability $\beta_\omega(s)$. The action set of an MDP can be augmented with options, which can be either handcrafted or automatically discovered. An MDP extended with options is a Semi-Markov Decision Process (SMDP); the learning methods for MDPs still apply to SMDPs. To improve sample-efficiency, the experiences $(s,a,r,s')$ generated by an option's policy can be used to update other options' policies. This transfer learning method is called \emph{intra-option learning} \cite{SuttonPS98}.

\subsection{Inductive Learning of Answer Set Programs}
In this section we describe answer set programming (ASP) and the ILASP system for learning ASP programs.

\subsubsection{Answer Set Programming (ASP)} \cite{gelfond_kahl_2014} is a declarative programming language for knowledge representation and reasoning. An ASP problem is expressed in a logical format and the models (called answer sets) of its representation provide the solutions to that problem.

A \emph{literal} is an atom $\mathtt{a}$ or its negation $\mathtt{not~a}$. Given an atom $\mathtt{h}$ and a set of literals $\mathtt{b_1,\ldots,b_n}$, a \emph{normal rule} is of the form $\mathtt{h \codeif b_1, \ldots, b_n}$, where $\mathtt{h}$ is the \emph{head} and $\mathtt{b_1,\ldots,b_n}$ is the \emph{body} of the rule. Rules of the form $\mathtt{\codeif b_1, \ldots, b_n}$ are called \emph{constraints}. In this paper, we assume an ASP program $P$ to be a set of normal rules and constraints.
Given a set of ground atoms (or \emph{interpretation}) $I$, a ground normal rule is satisfied if the head is satisfied by $I$ when the body literals are satisfied by $I$. A ground constraint is satisfied if the body is not satisfied by $I$. The \emph{reduct} $P^I$ of a program $P$ with respect to $I$ is built by removing all rules including $\mathtt{not~a}$ such that $\mathtt{a}\in I$ from $P$. $I$ is an \emph{answer set} of $P$ iff (1) $I$ satisfies the rules of $P^I$, and (2) no subset of $I$ satisfies the rules of $P^I$.

\subsubsection{ILASP} \cite{ILASP_system} is a system for learning ASP programs from  partial answer sets. A \emph{context-dependent partial interpretation (CDPI)} \cite{LawRB16} is a pair $\langle \langle e^{inc}, e^{exc} \rangle, e^{ctx} \rangle$, where $\langle e^{inc}, e^{exc}\rangle $ is a partial interpretation and $e^{ctx}$ is an ASP program, called \emph{context}. A program $P$ accepts $e$ iff there is an answer set $A$ of $P \cup e^{ctx}$ such that $e^{inc} \subseteq A$ and $e^{exc}~\cap A = \emptyset$. An \emph{ILASP task} \cite{LawRB16} is a tuple $T=\langle B,S_M,E\rangle$ where $B$ is the ASP background knowledge, $S_M$ is the set of rules allowed in the hypotheses, and $E$ is a set of CDPIs, called examples. A hypothesis $H \subseteq S_M$ is a \emph{solution} of $T$ iff $\forall e \in E$, $B \cup H$ accepts $e$.

\section{Methodology}
\label{sec:methodology}

In this section we describe {\methodname}, our method that interleaves the learning of subgoal automata with the learning of policies for achieving these subgoals. The tasks we consider are \emph{episodic MDPs} $\mathcal{M}=\langle S,A,p,r,\gamma, S_T, S_G\rangle$ where $r(s,a,s')=1$ if $s'\in S_G$ and 0 otherwise.

The automaton transitions are defined by a logical formula over a set of \emph{observables} $\mathcal{O}$. A \emph{labeling function} $L:S \to 2^\mathcal{O}$ maps an MDP state into a subset of observables $O \subseteq \mathcal{O}$ (or \emph{observations}) perceived by the agent at that state.

\begin{figure*}[ht]
	\subfloat[\label{fig:officeworld_grid}]{%
		\tikzset{digit/.style = { minimum height = 5mm, minimum width=5mm, anchor=center }}
		\newcommand{\setcell}[3]{\edef\x{#2 - 0.5}\edef\y{9.5 - #1}\node[digit,name={#1-#2}] at (\x, \y) {#3};}
		\centering
		\begin{tikzpicture}[scale=0.47]
		\draw[white] (0, 0) grid (13, 10);
		
		\setcell{9}{2}{0} \setcell{9}{3}{1} \setcell{9}{4}{2} \setcell{9}{5}{3} \setcell{9}{6}{4} \setcell{9}{7}{5} \setcell{9}{8}{6} \setcell{9}{9}{7} \setcell{9}{10}{8} \setcell{9}{11}{9} \setcell{9}{12}{10} \setcell{9}{13}{11}
		
		\setcell{8}{1}{0} \setcell{7}{1}{1} \setcell{6}{1}{2} \setcell{5}{1}{3} \setcell{4}{1}{4} \setcell{3}{1}{5} \setcell{2}{1}{6} \setcell{1}{1}{7} \setcell{0}{1}{8}
		
		\draw[gray] (1, 1) grid (13, 10);
		\draw[very thick, scale=3] (1/3, 1/3) rectangle (13/3, 10/3);
		
		\draw[very thick, scale=1] (4, 1) rectangle (4, 2); \draw[very thick, scale=1] (4, 9) rectangle (4, 10);
		\draw[very thick, scale=1] (4, 3) rectangle (4, 8);
		\draw[very thick, scale=1] (7, 1) rectangle (7, 2); \draw[very thick, scale=1] (7, 3) rectangle (7, 8); \draw[very thick, scale=1] (7, 9) rectangle (7, 10);
		\draw[very thick, scale=1] (10, 1) rectangle (10, 2); \draw[very thick, scale=1] (10, 3) rectangle (10, 8); \draw[very thick, scale=1] (10, 9) rectangle (10, 10);
		\draw[very thick, scale=1] (1, 4) rectangle (2, 4); \draw[very thick, scale=1] (3, 4) rectangle (11, 4); \draw[very thick, scale=1] (12, 4) rectangle (13, 4);
		\draw[very thick, scale=1] (1, 7) rectangle (2, 7); \draw[very thick, scale=1] (3, 7) rectangle (5, 7); \draw[very thick, scale=1] (6, 7) rectangle (8, 7); \draw[very thick, scale=1] (9, 7) rectangle (11, 7); \draw[very thick, scale=1] (12, 7) rectangle (13, 7);
		\setcell{1}{3}{$D$} \setcell{1}{6}{$\ast$} \setcell{1}{9}{$\ast$} \setcell{1}{12}{$C$}
		\setcell{2}{5}{\Coffeecup} \setcell{2}{6}{\Strichmaxerl[1.25]}
		\setcell{4}{3}{$\ast$} \setcell{4}{6}{$o$} \setcell{4}{9}{\Letter} \setcell{4}{12}{$\ast$}
		\setcell{7}{3}{$A$} \setcell{7}{6}{$\ast$} \setcell{7}{9}{$\ast$} \setcell{7}{12}{$B$}
		\end{tikzpicture}
	}
	\hfill
	\subfloat[\label{fig:officeworld_coffee_traces}]{%
		\shortstack[l]{
			Execution trace\\
			$\begin{aligned}
			T = \langle& s_{4,6}, \leftarrow, 0, s_{3,6}, \rightarrow, 0,\\ &s_{4,6},\downarrow, 0, s_{4,5}, \downarrow, 1, s_{4,4} \rangle
			\end{aligned}$\\~\\~\\
			Observation trace\\
			$\begin{aligned}
			T_{L,\mathcal{O}} = \left\langle \{\}, \{\text{\Coffeecup}\}, \{\}, \{\}, \{o\} \right\rangle
			\end{aligned}$\\~\\~\\
			Compressed observation trace\\
			$\begin{aligned}
			\hat{T}_{L,\mathcal{O}} = \left\langle \{\}, \{\text{\Coffeecup}\}, \{\}, \{o\} \right\rangle
			\end{aligned}$
		}
	}
	\hfill
	\subfloat[\label{fig:officeworld_coffee_rm}]{%
		\centering
		\begin{tikzpicture}[shorten >=1pt,node distance=2.06cm,on grid,auto]
		\node[state,initial] (u_0)   {$u_0$};
		\node[state,accepting] (u_acc) [below =of u_0]  {$u_{A}$};
		\node[state] (u_1) [left =of u_acc]   {$u_1$};
		\node[state] (u_rej) [right =of u_acc]  {$u_{R}$};
		
		\path[->] (u_0) edge [loop above] node {otherwise} ();
		\path[->] (u_1) edge [loop below] node {otherwise} ();
		
		\path[->] (u_0) edge [bend right] node [swap] {$\texttt{\Coffeecup} \land \neg o$} (u_1);
		\path[->] (u_0) edge [bend left] node {$\ast$} (u_rej);
		\path[->] (u_1) edge[bend right, below] node {$\ast$} (u_rej);
		\path[->] (u_1) edge[bend left, above] node [swap] {$o$} (u_acc);
		\path[->] (u_0) edge node {$\texttt{\Coffeecup} \land o$} (u_acc);
		\end{tikzpicture}
	}
	
	\caption{The \textsc{OfficeWorld} environment \cite{IcarteKVM18}. Figure \protect\subref{fig:officeworld_grid} is an example grid, \protect\subref{fig:officeworld_coffee_traces} shows a positive execution trace in the example grid for the \textsc{Coffee} task and its derived observation traces, and \protect\subref{fig:officeworld_coffee_rm} shows the \textsc{Coffee} automaton.}
	\label{fig:officeworld_and_rm}
\end{figure*}

We use the \textsc{OfficeWorld} environment \cite{IcarteKVM18} to explain our method. It consists of a grid (see Figure~\ref{fig:officeworld_grid}) where an agent ($\Strichmaxerl[1.25]$) can move in the four cardinal directions, and the set of observables is $\mathcal{O}=\lbrace\texttt{\Coffeecup}, \texttt{\Letter}, o, A, B,C,D,\ast\rbrace$. The agent picks up coffee and the mail at locations \Coffeecup~and \Letter~respectively, and delivers them to the office at location $o$. The decorations $\ast$ are broken if the agent steps on them. There are also four locations labeled $A$, $B$, $C$ and $D$. The observables $A$, $B$, $C$ and $D$, and decorations $\ast$ do not share locations with other elements. The agent observes these labels when it steps on their locations. Three tasks with different goals are defined on this environment: \textsc{Coffee} (deliver coffee to the office), \textsc{CoffeeMail} (deliver coffee and mail to the office), and \textsc{VisitABCD} (visit $A$, $B$, $C$ and $D$ in order). The tasks terminate when the goal is achieved or a $\ast$ is broken (this is a dead-end state).

\subsection{Traces}
An \emph{execution trace} $T$ is a finite state-action-reward sequence $T=\langle s_0,a_0,r_1,s_1,a_1, \ldots, a_{n-1},r_n, s_n\rangle$ induced by a (changing) policy during an episode. An execution trace is \emph{positive} if $s_n \in S_G$, \emph{negative} if $s_n \in S_T \setminus S_G$, and \emph{incomplete} if $s_n \notin S_T$, denoted by $T^+$, $T^-$ and $T^I$ respectively.

An \emph{observation trace} $T_{L,\mathcal{O}}$ is a sequence of observation sets $O_i\subseteq \mathcal{O}, 0 \leq i \leq n$, obtained by applying a labeling function $L$ to each state $s_i$ in an execution trace $T$. A \emph{compressed observation trace} $\hat{T}_{L,\mathcal{O}}$ is the result of removing contiguous duplicated observation sets from $T_{L, \mathcal{O}}$. Figure~\ref{fig:officeworld_coffee_traces} shows an example of a positive execution trace for the \textsc{Coffee} task together with the derived observation trace and the resulting compressed trace.

A set of execution traces is a tuple $\mathcal{T} = \langle\mathcal{T}^+, \mathcal{T}^-, \mathcal{T}^I \rangle$, where $\mathcal{T}^+$, $\mathcal{T}^-$ and $\mathcal{T}^I$ are sets of positive, negative and incomplete traces, respectively. The sets of observation and compressed observation traces are denoted $\mathcal{T}_{L,\mathcal{O}}$ and $\hat{\mathcal{T}}_{L,\mathcal{O}}$.

\subsection{Subgoal Automata}
A \emph{subgoal automaton} is a tuple $\mathcal{A}=\langle U, \mathcal{O}, \delta, u_0, u_A, u_R \rangle$ where $U$ is a finite set of states, $\mathcal{O}$ is a set of observables (or alphabet), $\delta: U \times 2^\mathcal{O} \to U$ is a deterministic transition function that takes as arguments a state and a subset of observables and returns a state, $u_0 \in U$ is a start state, $u_A \in U$ is the unique accepting state, and $u_R \in U$ is the unique rejecting state.

A subgoal automaton $\mathcal{A}$ \emph{accepts} an observation trace $T_{L,\mathcal{O}}=\langle O_0, \ldots, O_n\rangle$ if there exists a sequence of automaton states $u_0, \ldots, u_{n+1}$ in $U$ such that (1) $\delta(u_i, O_i)=u_{i+1}$ for $i=0,\ldots,n$, and (2) $u_{n+1} = u_A$. Analogously, $\mathcal{A}$ \emph{rejects} $T_{L,\mathcal{O}}$ if $u_{n+1} = u_R$.

When a subgoal automaton is used together with an MDP, the actual states are $(s,u)$ pairs where $s\in S$ and $u\in U$. Therefore, actions are selected according to a policy $\pi:S \times U \to \Delta(A)$ where $\pi(a|s,u)$ is the probability for taking action $a\in A$ at the MDP state $s\in S$ and the automaton state $u\in U$. At each step, the agent transitions to $(s',u')$ where $s'$ is the result of applying an action $a \in A$ in $s$, while $u'=\delta(u, L(s'))$. Then, the agent receives reward 1 if $u'=u_A$ and 0 otherwise.

Figure~\ref{fig:officeworld_coffee_rm} shows the automaton for the \textsc{Coffee} task of the \textsc{OfficeWorld} domain. Each transition is labeled with a logical condition $\varphi \in 3^\mathcal{O}$ that expresses the subgoal to be achieved\footnote{Note that $\varphi \in 3^\mathcal{O}$ because each observable can appear as positive or negative, or not appear in the condition.}. The sequence of visited automaton states for the trace in Figure~\ref{fig:officeworld_coffee_traces} would be $\langle u_0, u_1, u_1, u_1, u_A \rangle$.

\subsubsection{Relationship with options.}
Each state $u \in U$ in a subgoal automaton encapsulates an option $\omega_u=\langle I_{u}, \pi_{u}, \beta_{u} \rangle$ whose policy $\pi_{u}$ attempts to satisfy an outgoing transition's condition. Formally, the option termination condition $\beta_{u}$ is:
\begin{align*}
\beta_{u}(s) = \left\{\begin{matrix}
1~\text{if}~\exists u' \neq u, \varphi \in 3^\mathcal{O} \mathord{\mid} L(s') \models \varphi, \delta(u, \varphi) = u' \\
0~\text{otherwise}
\end{matrix}\right..
\end{align*}
That is, the option at automaton state $u \in U$ finishes at MDP state $s \in S$ if there is a transition to a different automaton state $u'\in U$ such that the transition's condition $\varphi$ is satisfied by the next observations $L(s')$. Note that at most one transition can be made true since the automaton is deterministic.

The initiation set $I_u$ is formed by all those states that satisfy the incoming conditions:
\begin{align*}
I_{u}=\left\lbrace s\in S \mid \delta\left(u', \varphi\right)=u, L\left(s\right) \models \varphi, u\neq u', u \neq u_0 \right\rbrace.
\end{align*}
In the particular case of the initial automaton state $u_0\in U$, its initiation set $I_{{u_0}}=S$ is the whole state space since there is no restriction imposed by any previous automaton state.

Note that we do not add options to the set of primitive actions, which would make the decision process more complex since more alternatives are available. Instead, subgoal automata keep the action set unchanged, and which option to apply is determined by the current automaton state.

\subsection{Learning Subgoal Automata from Traces}
This section describes our approach for learning an automaton. We formalize the automaton learning task as a tuple $T_{\mathcal{A}}=\langle U, \mathcal{O}, \mathcal{T}_{L,\mathcal{O}} \rangle $, where $U \supseteq \lbrace u_0, u_A, u_R\rbrace$ is a set of states, $\mathcal{O}$ is a set of observables, and $\mathcal{T}_{L,\mathcal{O}}=\langle \mathcal{T}^+_{L,\mathcal{O}},\mathcal{T}^-_{L,\mathcal{O}},\mathcal{T}^I_{L,\mathcal{O}}\rangle$ is a set of (possibly compressed) observation traces (abbreviated as \emph{traces} below). An automaton $\mathcal{A}$ is a solution of $T_\mathcal{A}$ if and only if accepts all positive traces $\mathcal{T}_{L,\mathcal{O}}^+$, rejects all negative traces $\mathcal{T}_{L,\mathcal{O}}^-$, and neither accepts nor rejects incomplete traces $\mathcal{T}_{L,\mathcal{O}}^I$.

The automaton learning task $T_{\mathcal{A}}$ is mapped into an ILASP task $M(T_{\mathcal{A}})=\langle B,S_M,E\rangle$. Then, the ILASP system is used to learn the smallest set of transitions (i.e.,~a minimal hypothesis) that covers the example traces. 

We define the components of the ILASP task $M(T_{\mathcal{A}})$ below. An ILASP task specifies the maximum size of a rule body. To allow for an arbitrary number of literals in the bodies, we learn the negation $\bar{\delta}$ of the actual transitions $\delta$. Thus, we do not limit the number of conjuncts, but the number of disjuncts (i.e., the number of edges between two states). We denote the maximum number of disjuncts by $\max|\delta(x,y)|$.	

\subsubsection{Hypothesis space.} The hypothesis space $S_M$ is formed by two kinds of rules:
\begin{enumerate}
	\item Facts of the form $\mathtt{ed(X, Y, E).}$ indicating there is a transition from state $\mathtt{X}\in U \setminus \lbrace u_A, u_R\rbrace$ to state $\mathtt{Y} \in U\setminus\lbrace \mathtt{X} \rbrace$ using edge $\mathtt{E} \in [1, \max|\delta(x,y)|]$.
	\item Normal rules whose \emph{head} is of the form $\mathtt{\bar{\delta}(X,Y,E,T)}$ stating that the conditions of the transition from state $\mathtt{X}\in U \setminus \lbrace u_A, u_R\rbrace$ to state $\mathtt{Y} \in U \setminus \lbrace \mathtt{X} \rbrace$ in edge $\mathtt{E} \in [1, \max|\delta(x,y)|]$ \emph{do not} hold at time $\mathtt{T}$. These conditions are specified in the \emph{body}, which is a conjunction of $\mathtt{obs(O,T)}$ literals indicating that observable $\mathtt{O}\in\mathcal{O}$ is seen at time $\mathtt{T}$. The atom $\mathtt{step(T)}$ expresses that $\mathtt{T}$ is a timestep. The body must contain at least one $\mathtt{obs}$ literal.
\end{enumerate}

Note that the hypothesis space does not include (i) loop transitions, and (ii) transitions from $u_A$ and $u_R$. Later, we define (i) in the absence of external transitions.

Given a subgoal automaton $\mathcal{A}$, we denote the set of ASP rules that describe it by $M(\mathcal{A})$. The rules below correspond to the $(u_0, u_1)$ and $(u_0, u_A)$ transitions in Figure \ref{fig:officeworld_coffee_rm}:

{\small\noindent
	\begin{tabular}{l}
		$\mathtt{ed(u_0, u_1, 1).~ed(u_0, u_A, 1).}$\\
		$\mathtt{\bar{\delta}(u_0, u_1, 1, T) \codeif not~obs(\text{\Coffeecup}, T), step(T).}$\\
		$\mathtt{\bar{\delta}(u_0, u_1, 1, T) \codeif obs(}o\mathtt{, T), step(T).}$\\
		$\mathtt{\bar{\delta}(u_0, u_A, 1, T) \codeif not~obs(}o\mathtt{, T), step(T).}$\\
		$\mathtt{\bar{\delta}(u_0, u_A, 1, T) \codeif not~obs(\text{\Coffeecup}, T), step(T).}$
	\end{tabular}
}

\subsubsection{Examples.} Given $\mathcal{T}_{L,\mathcal{O}}=\langle \mathcal{T}^+_{L,\mathcal{O}},\mathcal{T}^-_{L,\mathcal{O}},\mathcal{T}^I_{L,\mathcal{O}}\rangle$, the example set is defined as $E=\lbrace\langle e^*, C_{T_{L,\mathcal{O}}} \rangle \mid \ast \in \lbrace +,-,I \rbrace,$ $T_{L,\mathcal{O}} \in \mathcal{T}^*_{L,\mathcal{O}} \rbrace$, where $e^+=\langle \lbrace\mathtt{accept}\rbrace,\lbrace\mathtt{reject}\rbrace \rangle$, $e^-=\langle \lbrace\mathtt{reject}\rbrace, \lbrace\mathtt{accept}\rbrace \rangle$ and $e^I=\langle \lbrace\rbrace,\lbrace\mathtt{accept, reject}\rbrace \rangle$ are the partial interpretations for positive, negative and incomplete examples. The $\mathtt{accept}$ and  $\mathtt{reject}$ atoms express whether a trace is accepted or rejected by the automaton; hence, positive traces must only be accepted, negative traces must only be rejected, and incomplete traces cannot be accepted or rejected.

Given a trace $T_{L,\mathcal{O}}=\langle O_0,\ldots,O_n\rangle$, a context is defined as: $C_{T_{L,\mathcal{O}}} = \left\lbrace \mathtt{obs(O, T).} \mid \mathtt{O} \in O_{\mathtt{T}}, O_\mathtt{T} \in T_{L,\mathcal{O}}\right\rbrace \cup\left\lbrace\mathtt{ last(}n\mathtt{).}\right\rbrace$, where $\mathtt{last(}n\mathtt{)}$ indicates that the trace ends at time $n$. We denote by $M(T_{L,\mathcal{O}})$ the set of ASP facts that describe the trace $T_{L,\mathcal{O}}$. For example, $M(T_{L,\mathcal{O}})=\lbrace \mathtt{obs(a,0).}~\mathtt{obs(b,2).}~\mathtt{obs(c,2).}~\mathtt{last(2).}\rbrace$ for the trace $T_{L,\mathcal{O}}=\langle\lbrace a\rbrace,\lbrace\rbrace,\lbrace b,c\rbrace\rangle$.

\subsubsection{Background knowledge.} 
The next paragraphs describe the background knowledge $B$ components: $B = B_U \cup B_{\mathtt{step}} \cup B_\delta \cup B_\mathtt{st}$. First, $B_U$ is a set of facts of the form $\mathtt{state(u).}$ for each $\mathtt{u} \in U$. The set $B_\mathtt{step}$ defines a fluent $\mathtt{step(T)}$ for $0 \leq \mathtt{T} \leq \mathtt{T'} + 1$ where $\mathtt{T'}$ is the step for which $\mathtt{last(T')}$ is defined.

The subset $B_\delta$ defines rules for the automaton transition function. The first rule defines all the possible edge identifiers, which is limited by the maximum number of edges between two states. The second rule states that there is an external transition from state $\mathtt{X}$ at time $\mathtt{T}$ if there is a transition from $\mathtt{X}$ to a different state $\mathtt{Y}$ at that time. The third rule is a frame axiom: state $\mathtt{X}$ transitions to itself at time $\mathtt{T}$ if there are no external transitions from it at that time. The fourth rule defines the positive transitions in terms of the learned negative transitions $\bar{\delta}$ and $\mathtt{ed}$ atoms. The fifth rule preserves determinism: two transitions from $\mathtt{X}$ to two different states $\mathtt{Y}$ and $\mathtt{Z}$ cannot hold at the same time. The sixth rule forces all non-terminal states to have an edge to another state.
\begin{equation*}
B_\delta = {\small\begin{Bmatrix*}[l]
	\mathtt{edge\_id(1..\max |\delta(x,y)|)}.\\
	\mathtt{ext\_\delta(X, T)\codeif\delta(X, Y, \_, T), X !\mathord{=} Y.}\\
	\mathtt{\delta(X, X, 1, T) \codeif not~ext\_\delta(X, T), state(X), step(T).}\\
	\mathtt{\delta(X,Y,E,T) \codeif ed(X,Y,E), not~\bar{\delta}(X,Y,E,T), step(T).}\\
	\mathtt{\codeif \delta(X, Y, \_, T), \delta(X, Z, \_, T), Y !\mathord{=} Z.}\\
	\mathtt{\codeif not~ed(X,\_,\_), state(X), X!\mathord{=} u_A, X!\mathord{=} u_R.}
	\end{Bmatrix*}}
\end{equation*}

The subset $B_\mathtt{st}$ uses $\mathtt{st(T,X)}$ atoms, indicating that the agent is in state $\mathtt{X}$ at time $\mathtt{T}$. The first rule says that the agent is in $\mathtt{u_0}$ at time $\mathtt{0}$. The second rule determines that at time $\mathtt{T\mathord{+}1}$ the agent will be in state $\mathtt{Y}$ if it is in a non-terminal state $\mathtt{X}$ at time $\mathtt{T}$ and a transition between them holds. The third (resp.\ fourth) rule defines that the example is accepted (resp.\ rejected) if the state at the trace's last timestep is $\mathtt{u_A}$ (resp.\ $\mathtt{u_R}$).
\begin{equation*}
B_\mathtt{st} = {\small\begin{Bmatrix*}[l]
	\mathtt{st(0, u_0).}\\
	\mathtt{st(T\mathord{+}1,Y) \codeif st(T,X),\delta(X,Y,\_,T), X!\mathord{=} u_A, X!\mathord{=} u_R.}\\
	\mathtt{accept \codeif last(T), st(T\mathord{+}1, u_A).}\\
	\mathtt{reject \codeif last(T), st(T\mathord{+}1, u_R).}
	\end{Bmatrix*}}
\end{equation*}

Lemma \ref{lemma:asp_correctness} and Theorem \ref{theorem:first_theorem} capture the correctness of the encoding. We omit the proofs for brevity.

\begin{lemma}[Correctness of the ASP encoding]
	Given an automaton $\mathcal{A}$ and a finite trace $T_{L,\mathcal{O}}^\ast$, where $\ast \in \lbrace +, -, I\rbrace$, $M(\mathcal{A}) \cup B \cup M(T_{L,\mathcal{O}}^\ast)$ has a unique answer set $S$ and (1) $\mathtt{accept}\in S$ iff $\ast = +$, and (2) $\mathtt{reject} \in S$ iff $\ast = -$.
	\label{lemma:asp_correctness}
\end{lemma}

\begin{theorem}
	Given an automaton learning task $T_\mathcal{A}=\langle U, \mathcal{O},\mathcal{T}_{L,\mathcal{O}}\rangle$, an automaton $\mathcal{A}$ is a solution of $T_\mathcal{A}$ iff $M(\mathcal{A})$ is an inductive solution of $M(T_\mathcal{A})=\langle B, S_M, E \rangle$.
	\label{theorem:first_theorem}
\end{theorem}
%
%
%
%

\subsection{Interleaved Automata Learning Algorithm}
This section describes {\methodname} (\methodnamefullhighlight), a method that combines reinforcement and automaton learning. First, we describe the RL algorithm that exploits the automaton structure. Second, we explain how these two learning components are mixed.

\subsubsection{Reinforcement learning algorithm.} The RL algorithm we use to exploit the automata structure is QRM (Q-learning for Reward Machines) \cite{IcarteKVM18}. QRM maintains a Q-function for each automaton state, which are updated with Q-learning updates of the form:
\begin{align*}
Q_u\left(s,a\right) \xleftarrow{\alpha} r+\gamma \max_{a'}Q_{u'}\left(s',a'\right),
\end{align*}
where, in our case, $r=1$ if $u'=u_A$ and 0 otherwise. Note that the bootstrapped action-value depends on the next automaton state $u'$.

QRM performs this update for all the automaton states, so all policies are simultaneously updated  based on the $(s,a,s')$ experience. Note this is a form of intra-option learning \cite{SuttonPS98}: we update the policies of all the states from the experience generated by a single state's policy. In the tabular case, QRM is guaranteed to converge to an optimal policy in the limit. Note that QRM (and thus, {\methodname}) is still applicable in domains with large state spaces by having a Deep Q-Network \cite{MnihKSRVBGRFOPB15} in each automaton state instead of a Q-table.

\begin{algorithm}[ht]
	\caption{{\methodname} algorithm for a single task}
	\label{alg:interleaving_pseudocode}
	\begin{algorithmic}[1]
		\State $\mathcal{A} \leftarrow$ \textsc{InitAutomaton}($\lbrace u_0,u_A,u_R \rbrace$)
		\State $\mathcal{T}_{L,\mathcal{O}} \leftarrow \lbrace\rbrace$ \Comment Set of counterexamples
		\State \textsc{InitQFunctions}($\mathcal{A}$)
		\For{$l=0$ \textbf{to} num\_episodes}
		\State $s \leftarrow$ \textsc{EnvInitialState}()
		\State $u_p \leftarrow \delta(u_0, L(s))$
		\State $T_{L,\mathcal{O}} \leftarrow \langle L(s)\rangle$ \Comment Initialize trace
		\If {\textsc{IsCounterexample}($s, u_p$)}
		\State \textsc{OnCounterexampleFound}($T_{L,\mathcal{O}}$)
		\State $u_p \leftarrow \delta(u_0, L(s))$
		\EndIf
		\For{$t=0$ \textbf{to} length\_episode}
		\State $a, s' \leftarrow$ \textsc{EnvStep}($s, u_p$)
		\State $u_q \leftarrow \delta(u_p,L(s'))$
		\State \textsc{UpdateObsTrace}($L(s'),T_{L,\mathcal{O}}$)
		\If {\textsc{IsCounterexample}($s', u_q$)}
		\State \textsc{OnCounterexampleFound}($T_{L,\mathcal{O}}$)
		\State \algorithmicbreak
		\Else
		\State \textsc{UpdateQFunctions}($s$, $a$, $s'$, $L(s')$)
		\State $s \leftarrow s'; u_p \leftarrow u_q$
		\EndIf
		\EndFor
		\EndFor
		\Function{OnCounterexampleFound}{$T_{L,\mathcal{O}}$}
		\State $\mathcal{T}_{L,\mathcal{O}} \leftarrow \mathcal{T}_{L,\mathcal{O}} \cup \lbrace T_{L,\mathcal{O}} \rbrace$
		\State $\mathcal{A} \leftarrow$ \textsc{FindMinimalAutomaton}($\mathcal{T}_{L,\mathcal{O}}$)
		\State \textsc{InitQFunctions}($\mathcal{A}$)
		\EndFunction
	\end{algorithmic}
\end{algorithm}

\subsubsection{{\methodname} algorithm.} Algorithm~\ref{alg:interleaving_pseudocode} is the {\methodname} pseudocode for a single task and is explained below:
\begin{enumerate}
	\item The initial automaton (line 1) has initial state $u_0$, accepting state $u_A$ and rejecting state $u_R$. The automaton does not accept nor reject anything. The set of counterexample traces and the Q-functions are initialized (lines 2-3).
	\item When an episode starts, the current automaton state $u_p$ is $u_0$. One transition is applied depending on the agent's initial observations $L(s)$ (lines 5-6). In Figure~\ref{fig:officeworld_coffee_rm}, if the agent initially observes $\lbrace\text{\Coffeecup}\rbrace$, the actual initial state is $u_1$.
	\item At each step, we select an action $a$ in state $s$ using an $\epsilon$-greedy policy (line 12), and update the automaton state $u_p$ and observation trace $T_{L,\mathcal{O}}$ (lines 13-14).
	If no counterexample is found (line 18), the Q-functions of all automaton states are updated (line 19) and the episode continues. 
	\item Let $u$ be the current automaton state and $s$ the MDP state. A counterexample trace is found (lines 8, 15) if (a) multiple outgoing transitions from $u$ hold, or (b) the automaton does not correctly recognize $s$ (e.g., $s\in S_G \land u \neq u_A$).
	\item If a counterexample trace $T_{L,\mathcal{O}}$ is found (lines 21-24):
	\begin{enumerate}
		\item Add it to $\mathcal{T}_{L,\mathcal{O}}^+$ if $s\in S_G$, to $\mathcal{T}_{L,\mathcal{O}}^-$ if $s \in S_T \setminus S_G$ and to $\mathcal{T}_{L,\mathcal{O}}^I$ if $s \notin S_T$ (line 22).
		\item Run the automaton learner (line 23), using iterative deepening to select the number of automaton states.
		\begin{itemize}
			\item When a new automaton is learned, we reset all the Q-functions (e.g., setting all Q-values to 0) (line 24).
			\item If a counterexample is detected at the beginning of the episode (line 8), the automaton state is reset (line 10), else the episode ends (line 17).
		\end{itemize}
	\end{enumerate}
\end{enumerate}

{\methodname} does not start learning automata until we find a positive example (i.e., the goal is achieved). Resetting all the Q-functions causes the agent to forget everything it learned. To mitigate the forgetting effect and further exploit the automata structure, we employ reward shaping.

\subsubsection{Reward shaping.} \citeauthor{NgHR99}~(\citeyear{NgHR99}) proposed a function that provides the agent with additional reward to guide its learning process while guaranteeing that optimal policies remain unchanged:
\begin{align*}
F\left(s,a,s'\right)=\gamma\Phi\left(s'\right) - \Phi\left(s\right),
\end{align*}
where $\gamma$ is the MDP's discount factor and $\Phi:S\to\mathbb{R}$ is a real-valued function. The automata structure can be exploited by defining $F:U \times U \to \mathbb{R}$ in terms of the automaton states instead of the MDP states:
\begin{align*}
F\left(u, u'\right)=\gamma\Phi\left(u'\right) - \Phi\left(u\right),
\end{align*}
where $\Phi:U \to \mathbb{R}$. Intuitively, we want $F$'s output to be high when the agent gets closer to $u_A$. Thus, we define $\Phi$ as
\begin{align*}
\Phi\left(u\right) = \left|U\right| - d\left(u, u_A\right),
\end{align*}
where $|U|$ is the number of states in the automaton (an upper bound for the maximum finite distance between $u_0$ and $u_A$), and $d(u,u_A)$ is the distance between state $u$ and $u_A$. If $u_A$ is unreachable from a state $u$, then $d(u,u_A)=\infty$.

Theorem \ref{theorem:finite_steps} shows that if the target automata is in the hypothesis space, there will only be a finite number of learning steps in the algorithm before it converges on the target automata (or an equivalent automata).
\begin{theorem}
	Given a target finite automaton $\mathcal{A_*}$, there is no infinite sequence $\sigma$ of automaton-counterexample pairs $\langle\mathcal{A}_i, e_i\rangle$ such that $\forall i$: (1) $\mathcal{A}_i$ covers all examples $e_1, \ldots, e_{i-1}$, (2) $\mathcal{A}_i$ does not cover $e_i$, and (3) $\mathcal{A}_i$ is in the finite hypothesis space $S_M$.
	\label{theorem:finite_steps}
\end{theorem}
\begin{proof}
	By contradiction. Assume that $\sigma$ is infinite. Given that $S_M$ is finite, the number of possible automata is finite.  Hence, some automaton $\mathcal{A}$ must appear in $\sigma$ at least twice, say as $\mathcal{A}_i=\mathcal{A}_j, i<j$. By definition, $\mathcal{A}_i$ does not cover $e_i$ and $\mathcal{A}_j$ covers $e_i$. This is a contradiction.
\end{proof}

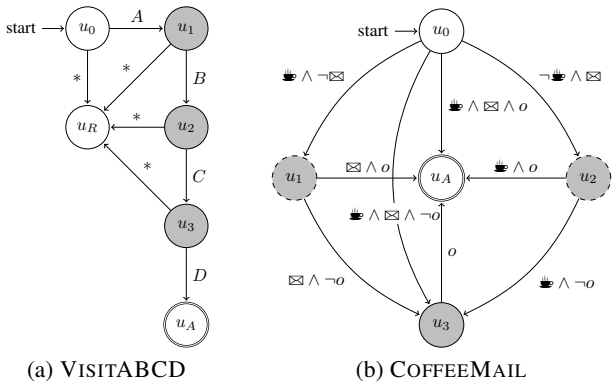
\begin{figure}[t]
	\centering
	\subfloat[\textsc{VisitABCD}\label{fig:isomorphisms_patrol}]{
		\resizebox{0.35\columnwidth}{!}{
			\begin{tikzpicture}[shorten >=1pt,node distance=2cm,on grid,auto]
			\node[state,initial] (u_0)   {$u_0$};
			\node[state,fill=lightgray] (u_1) [right =of u_0] {$u_1$};
			\node[state,fill=lightgray] (u_2) [below =of u_1] {$u_2$};
			\node[state,fill=lightgray] (u_3) [below =of u_2] {$u_3$};
			\node[state,accepting] (u_A) [below =of u_3] {$u_A$};
			\node[state] (u_R) [left =of u_2] {$u_R$};
			
			\path[->] (u_0) edge node {$A$} (u_1);
			\path[->] (u_1) edge node {$B$} (u_2);
			\path[->] (u_2) edge node {$C$} (u_3);
			\path[->] (u_3) edge node {$D$} (u_A);
			
			\path[->] (u_0) edge node [swap] {$\ast$} (u_R);
			\path[->] (u_1) edge node [swap] {$\ast$} (u_R);
			\path[->] (u_2) edge node [swap] {$\ast$} (u_R);
			\path[->] (u_3) edge node [swap] {$\ast$} (u_R);
			\end{tikzpicture}
		}
	}
	\hfill
	\subfloat[\textsc{CoffeeMail}\label{fig:isomorphisms_coffee_mail}]{
		\centering
		\resizebox{0.55\columnwidth}{!}{
			\begin{tikzpicture}[shorten >=1pt,node distance=2.9cm,on grid,auto]
			\node[state,initial] (u_0)   {$u_0$};
			\node[state,accepting] (u_A) [below =of u_0] {$u_A$};
			\node[state,fill=lightgray,dashed] (u_1) [left =of u_A] {$u_1$};
			\node[state,fill=lightgray,dashed] (u_2) [right =of u_A] {$u_2$};
			\node[state,fill=lightgray] (u_3) [below =of u_A] {$u_3$};
			
			\path[->] (u_0) edge[bend right=20] node [swap] {$\text{\Coffeecup} \land \neg\text{\Letter}$} (u_1);
			\path[->] (u_0) edge[bend left=20] node {$\neg\text{\Coffeecup} \land \text{\Letter}$} (u_2);
			\path[->] (u_0) edge[bend right, below] node[in place,pos=0.6] {$\text{\Coffeecup} \land \text{\Letter} \land \neg o$} (u_3);
			\path[->] (u_0) edge node {$\text{\Coffeecup} \land \text{\Letter} \land o$} (u_A);
			
			\path[->] (u_1) edge[bend right=20] node [swap] {$\text{\Letter} \land \neg o$} (u_3);
			\path[->] (u_1) edge[above] node [swap] {$\text{\Letter} \land o$} (u_A);
			
			\path[->] (u_2) edge[bend left=20] node {$\text{\Coffeecup} \land \neg o$} (u_3);
			\path[->] (u_2) edge[above] node {$\text{\Coffeecup} \land o$} (u_A);
			
			\path[->] (u_3) edge node [swap] {$o$} (u_A);
			\end{tikzpicture}
		}
	}
	\caption{Automata for two \textsc{OfficeWorld} tasks. Self-loops and transitions to $u_R$ in (b) are omitted. The shaded state IDs can be interchanged without $B_{sym}$. The dashed state IDs are still interchangeable even when using $B_{sym}$.
	}
	\label{fig:isomorphishms}
\end{figure}

\begin{figure*}[h]
	\subfloat[\textsc{Coffee}]{
		\includegraphics[width=0.32\textwidth]{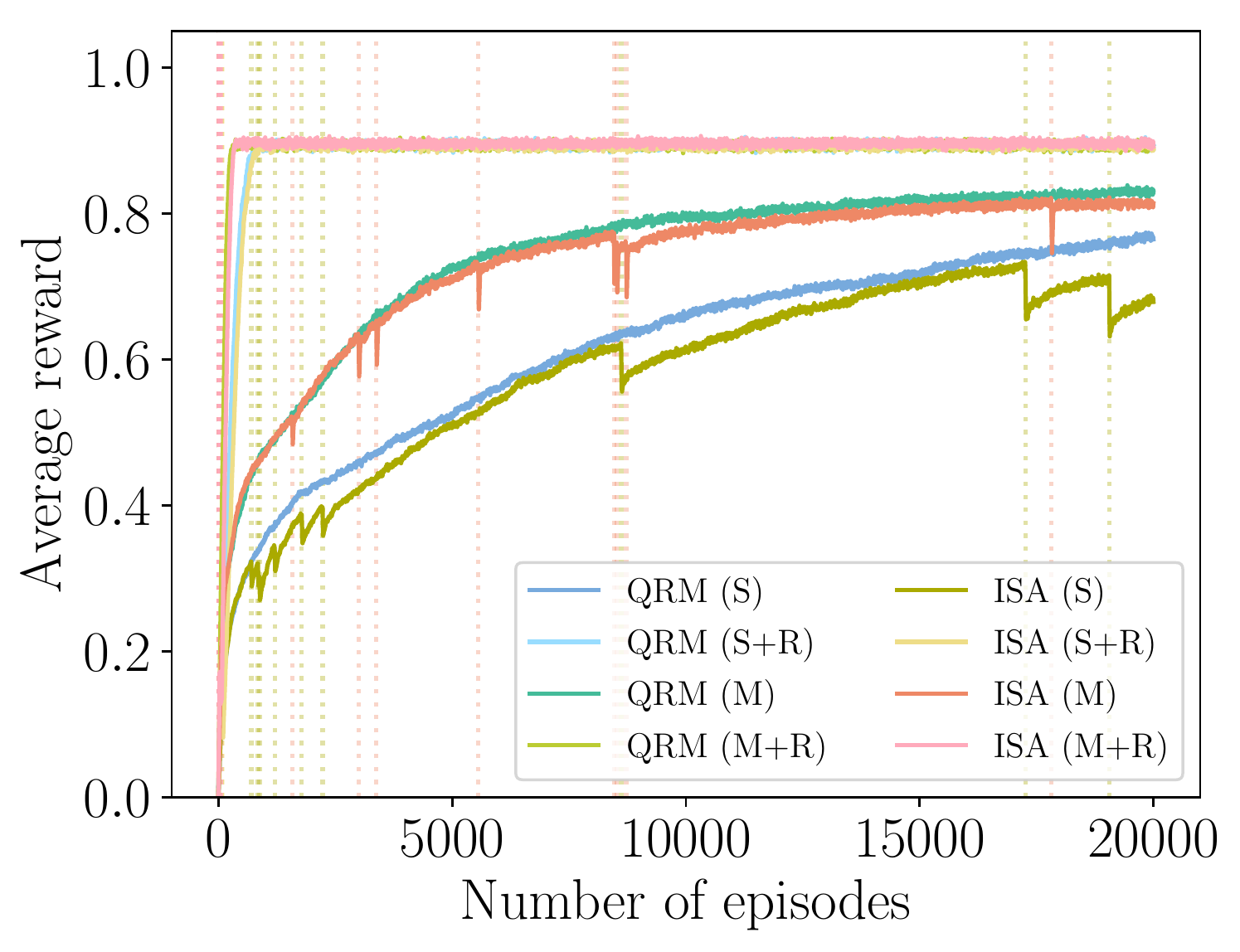}
	}
	\subfloat[\textsc{CoffeeMail}]{
		\includegraphics[width=0.32\textwidth]{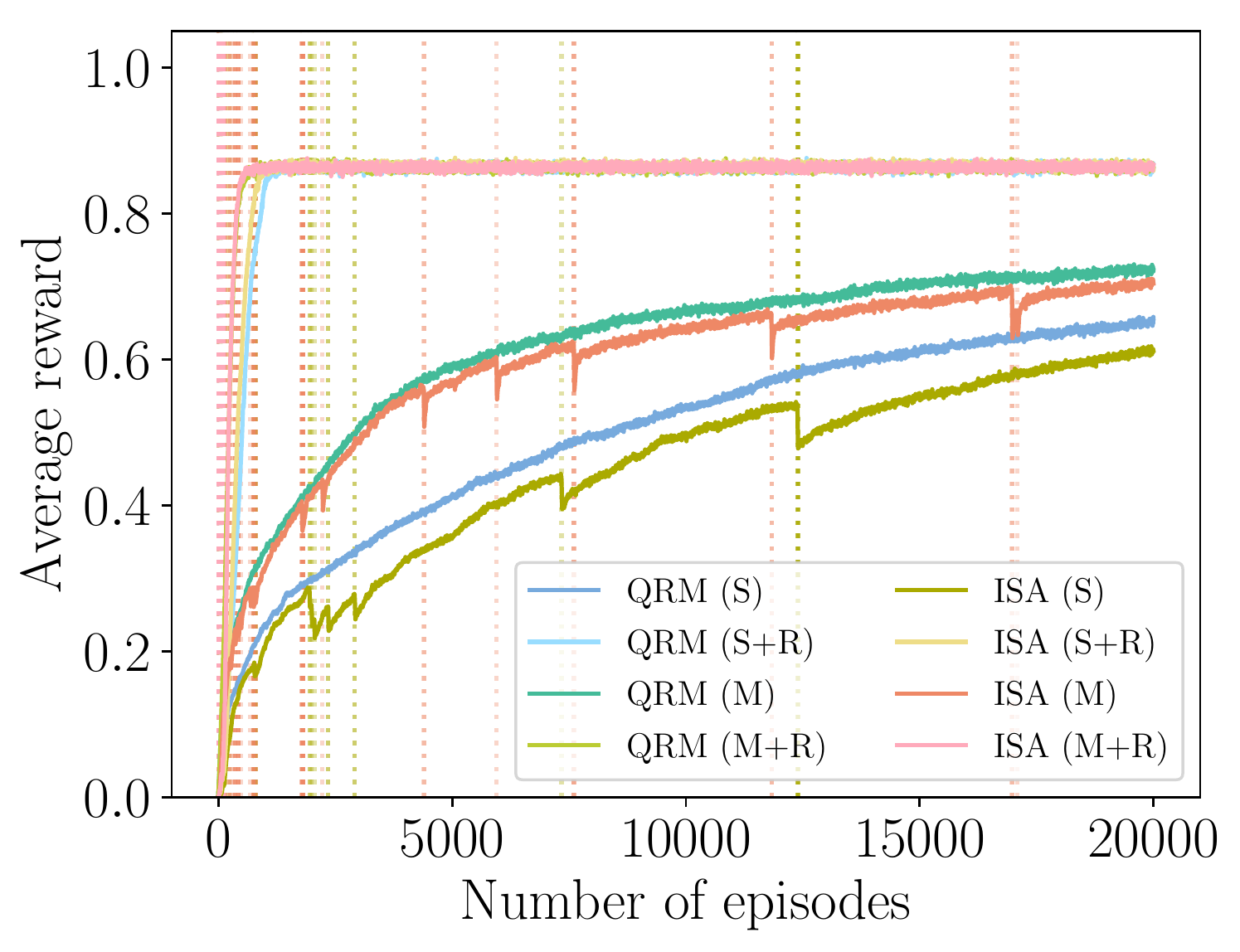}
	}
	\subfloat[\textsc{VisitABCD}]{
		\includegraphics[width=0.32\textwidth]{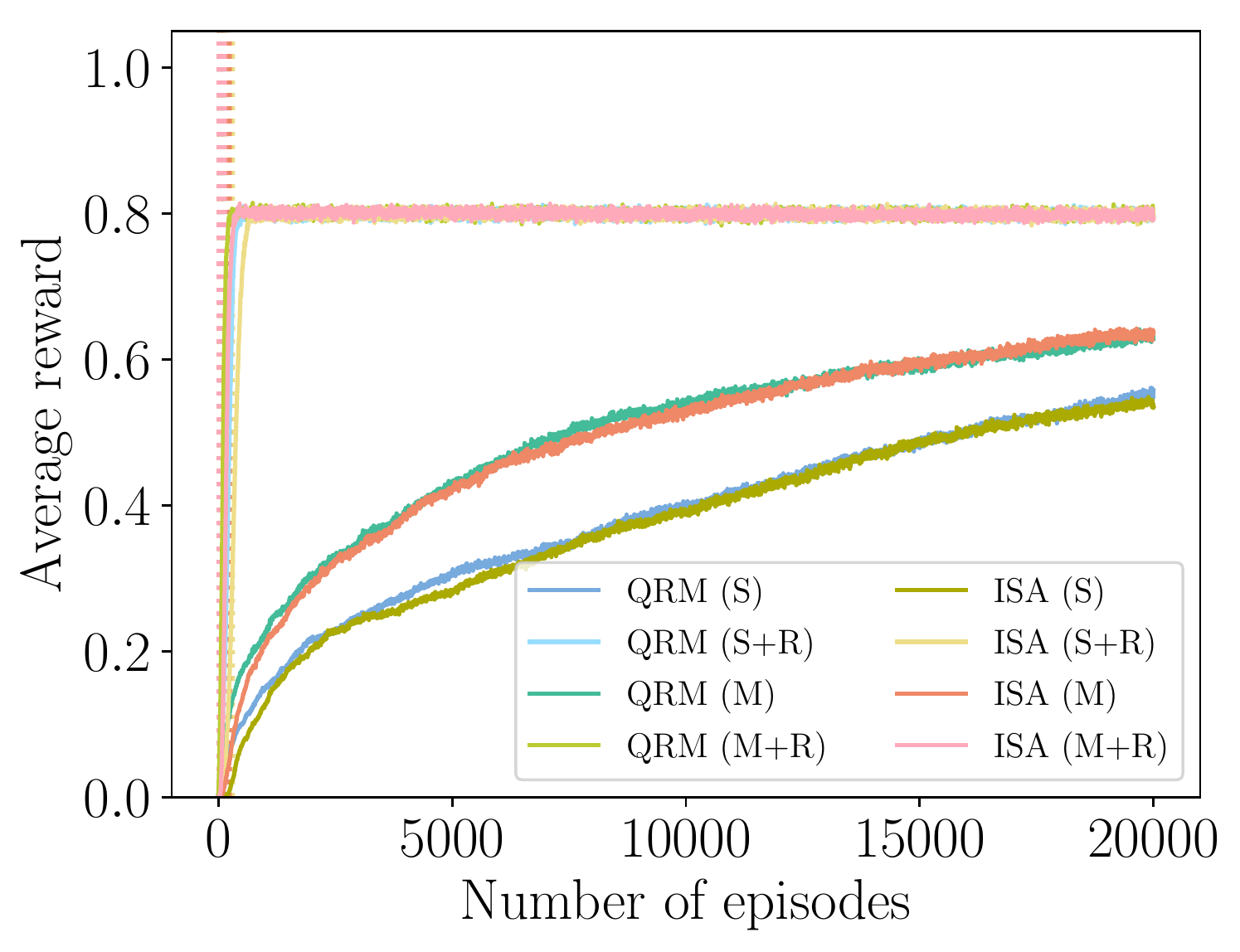}
	}
	\caption{Learning curves for different \textsc{OfficeWorld} tasks. The vertical lines are episodes where an automaton is learned.}
	\label{fig:learning_curves_officeworld}
\end{figure*}

\section{Experimental Results}
\label{sec:results}
We evaluate\footnote{Code: \url{github.com/ertsiger/induction-subgoal-automata-rl}.} {\methodname} using the \textsc{OfficeWorld} domain and the three tasks introduced in Section~\ref{sec:methodology}. The automata we compute using our method are forced to be \emph{acyclic}. Besides, we add a simple \emph{symmetry breaking} constraint $B_{sym}$ to the task's background knowledge to avoid considering isomorphic automata. Figure~\ref{fig:isomorphishms} shows two automata whose state IDs $u_1,u_2$ and $u_3$ can be used indifferently. Our symmetry breaking method (1) assigns an integer index to each state\footnote{$u_0$ has the lowest value, while $u_A$ and $u_R$ have the highest.} and (2) imposes that states must be visited in increasing order of indices. For example, if we assign indices 0\ldots3 to $u_0,\ldots,u_3$, positive traces in Figure~\ref{fig:isomorphisms_patrol} always yield the sequence $\langle u_0, u_1, u_2, u_3, u_A\rangle$. However, this is not enough to break all symmetries in Figure~\ref{fig:isomorphisms_coffee_mail}: $u_1$ and $u_2$ can still be switched since they cannot be in the same path to $u_A$ or $u_R$.

Tabular Q-learning is used to learn the Q-function at each automaton state with parameters $\alpha=0.1$, $\epsilon=0.1$, and $\gamma=0.99$. The agent's state is its position. {\methodname} receives a set of 100 randomly generated grids\footnote{Each grid has the same size and walls as Figure~\ref{fig:officeworld_grid}. The observables and the agent are randomly placed.}. One episode is run per grid in sequential order until reaching 20,000 episodes for each grid. The maximum episode length is 100 steps.
	
For some experiments we consider the multitask setting, where an automaton is learned for every task from the set of grids. Thus, there is a Q-table for each task-grid-automaton state triplet updated at every step. One episode is run per task-grid until 20,000 episodes are performed for each pair.

When reward shaping is on, the shaping function's output is set to -100 in case it is $-\infty$. This occurs when the next automaton state is $u_R$ since there is no path from it to $u_A$.

The different settings are referenced as \textbf{S} (single task), \textbf{M} (multitask) and \textbf{R} (reward shaping), all using the same set of 100 random grids. We execute 10 runs for each setting. 

We use ILASP to learn the automata from compressed observation traces and set $\max |\delta(x,y)|$ to 1. All experiments ran on 3.40GHz Intel\textsuperscript{\textregistered} Core\texttrademark~i7-6700 processors.

\subsubsection{{\methodname} performs similarly to QRM.}
Figure \ref{fig:learning_curves_officeworld} shows the tasks' average learning curve for ISA and QRM (where the automaton is given beforehand) across 10 runs. The {\methodname} curves converge similarly to the analogous QRM's. When reward shaping is used, convergence speeds up dramatically. The multitask settings converge faster since an agent is also trained from other agents' experiences in different tasks.

The vertical lines are episodes where an automaton is learned, and often occur during the first episodes: this is the main reason why the learning and non-learning curves are similar. Less frequently, automata learning also happens at intermediate phases in the \textsc{Coffee} and \textsc{CoffeeMail} tasks which make the agent forget what it learned. In these cases, recovery from forgetting happens faster in the multitask settings because of the reason above. Reward shaping has a positive effect on the learner: not only is convergence faster, it also helps the agent to discover helpful traces earlier.

\subsubsection{ISA's automata learning results.}
Table \ref{tab:ilasp_table_examples_number} shows the average number of examples needed to learn the final automata for setting \textbf{S} (the results for other settings are similar). Table \ref{tab:ilasp_table_time} shows the average time needed for computing \emph{all} the automata, which is negligible with respect to {\methodname}'s total running time. For both tables, the standard error is shown in brackets. First, we observe that the most complex tasks (\textsc{CoffeeMail} and \textsc{VisitABCD}) need a higher number of examples and more time to compute their corresponding automata. However, while the total number of examples for these tasks are similar, the time is higher for \textsc{VisitABCD} possibly because the longest path from $u_0$ to $u_A$ is longer than in \textsc{CoffeeMail}. Second, we see that the number of positive and incomplete examples are usually the smallest and the largest respectively. Note that the number of positive examples is approximately equal to the number of paths from $u_0$ to $u_A$. The paths described by the positive examples are refined through the negative and incomplete ones. Finally, Table \ref{tab:ilasp_table_time} shows slight variations of time across settings. The different experimental settings and the exploratory nature of the agent are responsible for coming up with different counterexamples and, thus, cause these variations. There is not a clear setting which leads to faster automata learning across domains. The design of exploratory strategies that speed up automata learning is possible future work.

\begin{table}[t]
	\centering
	\resizebox{\columnwidth}{!}{%
		\begin{tabular}{lrrrr}
			\toprule
			                    & \multicolumn{1}{c}{All} & \multicolumn{1}{c}{+} & \multicolumn{1}{c}{-} & \multicolumn{1}{c}{I} \\
			\midrule
			\textsc{Coffee}     & 6.6 (0.5)               & 2.2 (0.2)             & 2.3 (0.2)             & 2.1 (0.3)             \\
			\midrule
			\textsc{CoffeeMail} & 34.5 (2.9)              & 5.5 (0.4)             & 9.9 (0.9)             & 19.1 (2.2)            \\
			\midrule
			\textsc{VisitABCD}  & 32.5 (2.1)              & 1.7 (0.2)             & 11.6 (0.8)            & 19.2 (1.7)            \\
			\bottomrule
		\end{tabular}
	}
	\caption{Average number of examples needed for setting \textbf{S}.}
	\label{tab:ilasp_table_examples_number}
\end{table}

\begin{table}[t]
	\centering
	\resizebox{\columnwidth}{!}{%
		\begin{tabular}{lrrrr}
			\toprule
			                    & \multicolumn{1}{c}{\textbf{S}} & \multicolumn{1}{c}{\textbf{S+R}} & \multicolumn{1}{c}{\textbf{M}}& \multicolumn{1}{c}{\textbf{M+R}} \\
			\midrule
			\textsc{Coffee}     & 0.5 (0.0)                      & 0.4 (0.0)                        & 0.3 (0.0)                     & 0.4 (0.0)               \\
			\midrule
			\textsc{CoffeeMail} & 43.3 (12.1)                    & 36.9 (6.0)                       & 24.8 (3.6)                    & 24.6 (2.7)              \\
			\midrule
			\textsc{VisitABCD}  & 63.0 (11.4)                    & 68.5 (13.0)                      & 48.4 (8.8)                    & 69.6 (8.1)              \\
			\bottomrule
		\end{tabular}
	}
	\caption{Average ILASP running time.}
	\label{tab:ilasp_table_time}
\end{table}

\subsubsection{{\methodname} learns automata faster with few observables.} To test the effect of the observable set $\mathcal{O}$ on ILASP's performance, we run the experiments using setting \textbf{S} but employing only the observables that each task needs, e.g., $\mathcal{O}=\{\text{\Coffeecup}, o, \ast\}$ in the \textsc{Coffee} task. The biggest changes occur in \textsc{CoffeeMail}, where the total number of examples is 46\% smaller. The sets of positive, negative and incomplete examples are 27\%, 42\% and 53\% smaller. Besides, the automata are computed 92\% faster. Thus, we see that the number of observables has an impact on the performance: the RL process is halted less frequently and automata are learned faster. This performance boost in \textsc{CoffeeMail} can be due to the fact that it has more paths to $u_A$; thus, irrelevant symbols do not need to be discarded for each of these. This is confirmed by the fact that while the number of positives is roughly the same, the number of incomplete examples greatly decreases.

\section{Related Work}
\label{sec:related_work}

\subsubsection{Hierarchical RL (HRL).} Our method is closely related to the options framework for HRL, and indeed we define one option per automaton state. The key difference from other HRL approaches, like HAMs \cite{ParrR97}, MAXQ \cite{Dietterich00} and the common way of using options, is that we do not learn a high-level policy for selecting among options. Rather, the high-level policy is implicitly represented by the automaton, and the option to execute is fully determined by the current automaton state. Consequently, while HRL policies may be suboptimal in general, the QRM algorithm we use converges to the optimal policy.

Our approach is similar to HAMs in that they also use an automaton. However, HAMs are non-deterministic automata whose transitions can invoke lower level machines and are not labeled by observables (the high-level policy consists in deciding which transition to fire). \citeauthor{LeonettiIP12}~(\citeyear{LeonettiIP12}) synthesize a HAM from the set of shortest solutions to a non-deterministic planning problem, and use it to refine the choices at non-deterministic points through RL.

\subsubsection{Option discovery.} ISA is similar to bottleneck option discovery methods, which find ``bridges'' between regions of the state space. In particular, ISA finds conditions that connect two of these regions. \citeauthor{McGovernB01}~(\citeyear{McGovernB01}) use diverse density to find landmark states in state traces that achieve the task's goal. This approach is similar to ours because (1) it learns from traces; (2) it classifies traces into different categories; and (3) it interleaves option discovery and learning.

Just like some option discovery methods \cite{McGovernB01,StolleP02}, our approach requires the task to be solved at least once. Other methods \cite{MenacheMS02,SimsekB04,SimsekWB05,MachadoBB17} discover options without solving the task.

Grammars are an alternative to automata for expressing formal languages. \citeauthor{LangeF19}~(\citeyear{LangeF19}) induce a straight-line grammar from action traces to discover macro-actions.

\subsubsection{Reward machines (RMs).} Subgoal automata are similar to RMs \cite{IcarteKVM18}. There are two differences: (1) RMs do not have explicit accepting and rejecting states, and (2) RMs use a reward-transition function $\delta_r:U\times U\to \mathbb{R}$ that returns the reward for taking a transition between two automaton states. Note that our automata are a specific case of the latter where $\delta_r(\cdot, u_A) = 1$ and 0 otherwise.

Recent work has focused on learning RMs from experience using discrete optimization \cite{IcarteWKVCM19_NeurIPS} and grammatical inference algorithms \cite{XhuGAMNTW19}. While these approaches can get stuck in a local optima, ISA returns a minimal automaton each time it is called. Just as our method, they use an `a priori' specified set of observables.

In contrast to the other learning approaches, in the future we can leverage ILP to (1) transfer knowledge between automata learning tasks (e.g., providing a partial automaton as the background knowledge), (2) support examples generated by a noisy labeling function, and (3) easily express more complex conditions (e.g., using first-order logic).

\citeauthor{CamachoIKVM19}~(\citeyear{CamachoIKVM19}) convert reward functions expressed in various formal languages into RMs, and propose a reward shaping method that runs value iteration on the RM states.

\section{Conclusions and Future Work}
\label{sec:conclusions}
In this paper we have proposed {\methodname}, an algorithm for learning subgoals by inducing a deterministic finite automaton from observation traces seen by the agent. The automaton structure can be exploited by an existing RL algorithm to increase sample efficiency and transfer learning. We have shown that our method performs comparably to an algorithm where the automaton is given beforehand. 

Improving the scalability is needed to handle more complex tasks requiring automata with cycles or longer examples. In the future, we will further explore symmetry breaking techniques to reduce the hypothesis space \cite{DrescherTW11} and other approaches automata learning, like RNNs \cite{WeissGY18,Michalenko19}. Discovering the observables used by the automata is also an interesting path for future work.

\section*{Acknowledgments}
Anders Jonsson is partially supported by the Spanish grants TIN2015-67959 and PCIN-2017-082.

\fontsize{9pt}{10pt} \selectfont
\bibliographystyle{aaai}
\bibliography{aaai20.bib}

\end{document}